\documentclass{article} 
\usepackage{iclr2027_conference,times}

\usepackage{amsmath,amsfonts,bm}

\def\1{\bm{1}}

\def\vtheta{{\bm{\theta}}}

\def\vx{{\bm{x}}}
\def\vy{{\bm{y}}}

\def\mI{{\bm{I}}}

\def\mK{{\bm{K}}}

\def\mX{{\bm{X}}}

\DeclareMathAlphabet{\mathsfit}{\encodingdefault}{\sfdefault}{m}{sl}
\SetMathAlphabet{\mathsfit}{bold}{\encodingdefault}{\sfdefault}{bx}{n}

\def\gD{{\mathcal{D}}}
\def\gE{{\mathcal{E}}}

\def\gN{{\mathcal{N}}}

\def\gX{{\mathcal{X}}}

\newcommand{\E}{\mathbb{E}}

\newcommand{\R}{\mathbb{R}}

\DeclareMathOperator*{\argmax}{arg\,max}

\newcommand{\D}{\mathcal{D}}
\newcommand{\X}{\mathcal{X}}
\newcommand{\xs}{\vx_*}

\usepackage{hyperref}
\usepackage{url}

\title{Work While They Sleep: Exploiting Evaluation Latency for Fully Bayesian Optimization}

\iclrfinalcopy

\author{\textbf{Gustavo Sutter}$^{1,2}$\thanks{Equal contribution.} \quad
\textbf{Alejandro Comas-Leon}$^{1,2}$\footnotemark[1] \quad
\textbf{David Holzmüller}$^{3}$\quad
\textbf{Hao Wang}$^{4}$ \\
\textbf{Luis Ricardez-Sandoval}$^{4,5}$ \quad
\textbf{Pascal Poupart}$^{1,2}$ \quad
\textbf{Agustinus Kristiadi}$^{2,6}$\\ [4pt]
$^1$Cheriton School of Computer Science, University of Waterloo, Waterloo, ON, Canada \\
$^2$Vector Institute, Toronto, ON, Canada  \\
$^3$SODA Team, INRIA Saclay, Palaiseau, France \\
$^4$Department of Chemical Engineering, University of Waterloo, Waterloo, ON, Canada \\
$^5$Waterloo Institute for Nanotechnology, University of Waterloo, Waterloo, ON, Canada \\
$^6$Department of Computer Science, Western University, London, ON, Canada
}

\usepackage{graphicx} 
\usepackage{xcolor}
\usepackage[capitalise]{cleveref}
\usepackage{algorithm}
\usepackage{algpseudocode}
\usepackage{wrapfig}
\usepackage{tikz}
\usepackage{subcaption}

\definecolor{mydarkblue}{rgb}{0,0.08,0.45}
\hypersetup{
  colorlinks=true,
  citecolor=mydarkblue,
  linkcolor=mydarkblue,
  urlcolor=mydarkblue
}

\usepackage{amsthm}
\usepackage{thm-restate}

\theoremstyle{plain}

\theoremstyle{definition}

\theoremstyle{remark}

\newcommand{\mut}{\mu^{(\vtheta)}}
\newcommand{\kt}{k^{(\vtheta)}}
\newcommand{\kappat}{\mathbf{\kappa}^{(\vtheta)}}

\definecolor{tabblue}{HTML}{1f77b4}
\definecolor{tabred}{HTML}{d62728}
\definecolor{tabgreen}{HTML}{2ca02c}
\definecolor{tabgray}{HTML}{7f7f7f}

\newcommand{\random}{\textcolor{tabgray}{Random}}
\newcommand{\map}{\textcolor{tabblue}{MAP}}
\newcommand{\mcmc}{\textcolor{tabgreen}{Standard FBO}}
\newcommand{\elfbo}{\textcolor{tabred}{ELF-BO}}

\definecolor{idleblue}{HTML}{1F5FA9}

\begin{document}

\maketitle

\begin{abstract}
Black-box optimization problems are ubiquitous across science and engineering, often dealing with expensive objective functions.
This objective latency has two consequences during optimization: (i) the objective evaluation dominates execution time, and (ii) sample-efficient algorithms are crucial to accelerate development and avoid wasting resources.
Bayesian optimization (BO) methods are the \textit{de facto} choice of planners for suggesting the next point to try.
Standard BO fits the surrogate model's hyperparameters with a point estimate.
Alternatively, a fully Bayesian approach uses model averaging to account for uncertainty over the hyperparameters, leading to better uncertainty estimates---useful in the low-data regime that is pervasive in BO.
However, it is often prohibitively expensive and thus rarely used.
In this work, we propose ELF-BO, an algorithm that uses the objective evaluation latency to headstart the computation of the next suggestion, allowing for fully Bayesian optimization without incurring substantial decision-time costs.
This is done by sampling from the hyperparameter posterior \emph{while} the objective is being evaluated, only requiring reweighting of the samples once the objective value is observed.
Across synthetic functions and real-world applications, we show that ELF-BO matches the performance of fully Bayesian methods while only incurring decision latency on par with or better than standard BO.
Thus, ELF-BO makes fully Bayesian optimization practical in real-world use cases.
\end{abstract}

\section{Introduction}

Bayesian optimization~\citep[BO;][]{Mockus1975} often acts as a planner (suggesting the next query point to evaluate) for optimizing black-box functions commonly encountered in science \citep{greenaway2023alchemist, Muthyala2026bodenovo} and engineering~\citep{lam2018boaerospace, ament2023sustainable}.
Across domains, most applications of BO deal with expensive objective functions, a property that motivates the development of sample-efficient methods.
The inherent high evaluation latency makes it the main bottleneck in the optimization run.
Meanwhile, after suggesting a query location, the planner waits idly for the evaluation results before continuing.
Crucially, expensive function evaluations also limit applications to cases with restricted data availability, both in initial dataset size and in the number of evaluations allowed in the optimization loop \citep{Jones1998EfficientGO}.
The former increases the difficulty of achieving the latter, as a misspecified model can lead to suboptimal queries that waste precious evaluation budget.

Standard BO methods utilize Gaussian processes~\citep[GP;][]{Rasmussen_Williams_gp_2006} as surrogate models and use a point estimate approach---either maximum likelihood estimation (MLE) or maximum a posteriori (MAP)---to find the GP's hyperparameters (e.g., lengthscale, output scale, deep-kernel parameters, and noise variance).
However, relying on a single estimate of the hyperparameters can drastically underestimate the uncertainty over the true hyperparameters, and therefore the uncertainty over the objective values \citep{death2021howbayesian}.
Importantly, this is a bigger problem in the low-data regimes where well-calibrated models are required for sample-efficient decision-making.
The same trade-off arises beyond GPs: the high computational overhead of Bayesian neural networks~\citep[BNNs;][]{li2024study,kristiadi_sober_2024} makes them similarly expensive at decision time.

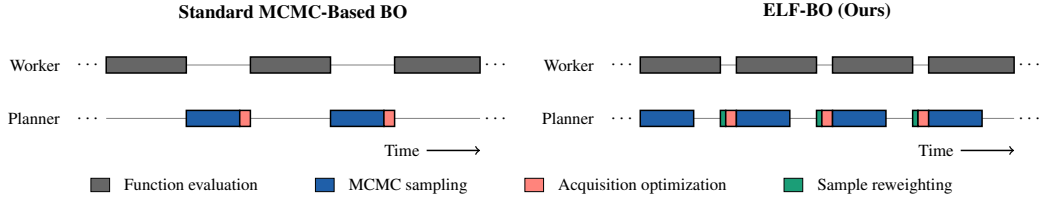
\begin{figure}[t]
\centering
\definecolor{dark2gray}{HTML}{666666}
\definecolor{dark2pink}{HTML}{FF847C}
\definecolor{dark2green}{HTML}{1B9E77}
\definecolor{dark2blue}{HTML}{1F5FA9}

\tikzset{
    task/.style={draw, minimum height=0.6cm},
    worker/.style={font=\small},
    swatch/.style={draw, minimum width=0.35cm, minimum height=0.25cm, inner sep=0pt},
    type1/.style={fill=dark2gray},
    type2/.style={fill=dark2blue},
    type3/.style={fill=dark2pink},
    type4/.style={fill=dark2green}
}
\newcommand{\lbox}[1]{\tikz[]{\node[swatch,#1]{};}}

\resizebox{0.99\textwidth}{!}{%
\begin{tikzpicture}[thick]

\newcommand{\schedule}[5][7]{%
  \begin{scope}[xshift=#3]
    \edef\tasklist{#5}%
    \foreach \name [count=\nw] in {#4} {\xdef\nworkers{\nw}}%
    \foreach \name [count=\w] in {#4} {
      \pgfmathsetmacro{\y}{\nworkers-\w+0.5}
      \draw[gray, thin] (0,\y) -- (#1,\y);
      \node[anchor=east, inner sep=2pt] (lell) at (0,\y) {$\cdots$};
      \node[anchor=west, inner sep=2pt]        at (#1,\y) {$\cdots$};
      \node[worker, anchor=east] at ([xshift=-0.1cm]lell.west) {\name};
    }
    \pgfmathsetmacro{\half}{#1/2}
    \draw[->] ({#1-1},-0.1) node[worker, left] {Time} -- (#1,-0.1);
    \node[font=\bfseries] at (\half,\nworkers+0.5) {#2};
    \foreach \wrk/\st/\en/\ty in \tasklist {
      \pgfmathsetmacro{\y}{\nworkers-\wrk+0.5}
      \draw[task, type\ty] (\st,\y-0.15) rectangle (\en,\y+0.15);
    }
  \end{scope}%
}

               

\def\scheduleA{1/0/1.5/1, 2/1.5/2.5/2, 2/2.5/2.7/3, 
               1/2.7/4.2/1, 2/4.2/5.2/2, 2/5.2/5.4/3, 
               1/5.4/7/1}
               
\def\scheduleB{1/0/1.5/1, 2/0/1/2, 2/1.5/1.6/4, 2/1.6/1.8/3,
               1/1.8/3.3/1, 2/1.8/2.8/2, 2/3.3/3.4/4, 2/3.4/3.6/3,
               1/3.6/5.1/1, 2/3.6/4.6/2, 2/5.1/5.2/4, 2/5.2/5.4/3,
               1/5.4/7/1, 2/5.4/6.4/2}

\schedule{Standard MCMC-Based BO}{0cm}{Worker, Planner}{\scheduleA}
\schedule{ELF-BO (Ours)}{10cm}{Worker, Planner}{\scheduleB}

\matrix[anchor=north, ampersand replacement=\&, column sep=0.15cm,
        nodes={anchor=center, thin},
        legendtext/.style={font=\small, text height=1.5ex, text depth=0.25ex, inner ysep=0pt}]
  at ([yshift=-0.15cm]current bounding box.south) {
    \node[swatch, type1] {}; \& \node[legendtext] {Function evaluation};      \&[0.8cm]
    \node[swatch, type2] {}; \& \node[legendtext] {MCMC sampling};            \&[0.8cm]
    \node[swatch, type3] {}; \& \node[legendtext] {Acquisition optimization}; \&[0.8cm]
    \node[swatch, type4] {}; \& \node[legendtext] {Sample reweighting};       \\
};

\end{tikzpicture}%
}
\caption{
\textbf{Where the time goes in a sequential BO campaign.}
Each timeline shows the \emph{worker} (the experiment or simulator evaluating $f$) and the \emph{planner} (the BO algorithm); thick segments denote busy time.
\textbf{Left:} In standard fully Bayesian BO, the two alternate: the planner waits while $f$ is evaluated, and the worker waits while the planner runs MCMC and optimizes the acquisition function.
\textbf{Right:} ELF-BO runs MCMC \emph{while} $f$ is being evaluated, so once $y_t$ arrives, only a cheap importance reweighting and the acquisition optimization remain, greatly reducing decision latency.
}
\label{fig:schedules}
\end{figure}

An alternative to point estimates is the use of a fully Bayesian optimization \citep{osborne2009gaussian,benassi2011robust}, in which the hyperparameters are averaged out according to their posterior distribution.
In practice, this is done by using Markov chain Monte Carlo~\citep[MCMC;][]{neal2011hmc} methods to obtain samples from the posterior distribution and computing the average of the acquisition function over the samples.
As usual in Bayesian inference, the price is paid with more computation and time, mostly due to the expensive MCMC sampling, rendering fully Bayesian methods prohibitively expensive for many practitioners.
This expensive decision-making step between objective function evaluations is undesirable in various applications; e.g., in automated materials discovery, machine time is valuable, so minimal downtime is desired ~\citep{MacLeod2020selfdrivinglab}.

In this work, we show that most of the expensive steps needed in a fully Bayesian approach can be performed while the expensive objective function is being evaluated.
We propose \textbf{E}valuation \textbf{L}atency \textbf{F}ully \textbf{B}ayesian \textbf{O}ptimization (\textbf{ELF-BO}), which uses the BO planner's otherwise idle time to reduce the footprint of fully Bayesian optimization after the evaluation result becomes available---see \cref{fig:schedules}.
This is accomplished by showing that an expectation over the next hyperparameter posterior can be computed using a weighted expectation with the posterior of the hyperparameters at the current step.
These weights are proportional to the posterior predictive density of the pending evaluation results---thus only the weights are computed once the evaluation is done.
Since this identity relies only on Bayes' rule, ELF-BO is not tied to GPs and can be applied to BNNs. 
In this work, however, we focus on GPs due to their prevalence and simplicity in BO.

We evaluate ELF-BO on 15 standard benchmarks ranging from 2 to 10 dimensions and on three real-world experimental datasets.
Across these problems, ELF-BO matches the optimization performance of na\"{i}ve MCMC-based fully Bayesian BO after every observation, while its decision latency---the time the experiment sits idle waiting for the next query---stays close to that of point-estimate BO.
The latter amounts to a decision latency roughly 10$\times$ lower than standard fully Bayesian BO with a Matérn kernel, 70--110$\times$ lower with a deep kernel, and more than two orders of magnitude lower as the dataset grows.
We further use the effective sample size of the importance weights to diagnose when reweighting is reliable, and characterize the regimes in which ELF-BO yields the largest gains.
Altogether, ELF-BO makes MCMC-based fully Bayesian BO practical in real-world scenarios---the more expensive the function evaluation, the more calibrated the planner will be, without high compute overhead.
\section{Preliminaries}

We consider the maximization problem \(\xs \in \argmax_{\vx \in \X}f(\vx)\) where \(f:\X \to \R\) is an expensive black-box objective function that can only be accessed via noisy observations \(y_i=f(\vx_i) + \epsilon_i\), with \(\epsilon_i \sim \mathcal{N}(0,\sigma_n^2)\).
Bayesian optimization (BO) provides a principled solution to this problem grounded in Bayesian decision theory.
The key components of BO are: a surrogate model \(p(f \mid \D_t)\) and an acquisition function \(\alpha(\vx;\D_t)\), where \(\D_t = \{(\vx_i, y_i)\}_{i=1}^t\) is the set of observations of the objective function up to the current time step $t \in \{ 1, \dots, T \}$.
The surrogate models the current belief over the objective function given the data observed so far, which the acquisition function then uses to score the value of querying each point next.
The next query is chosen as
\(\vx_{t+1} \in \argmax_{\vx \in \X} \alpha(\vx ; \D_t)\), the objective is
evaluated at \(\vx_{t+1}\), and the procedure repeats until the evaluation
budget is exhausted.

\subsection{Gaussian Processes}
\label{sec:gp}

The standard surrogate choice for BO is the Gaussian process.
A Gaussian process \(\mathcal{GP}(\mut, \kt)\) is a distribution over functions fully characterized by its prior mean function \(\mut:\gX \to \R\) and its covariance function \(\kt:\gX \times \gX \to \R\),
where \(\vtheta\) collects all hyperparameters of the mean and covariance functions, together with the observation noise variance \(\sigma_n^2\).

Conditioned on an observation dataset \(\D_t = \{(\vx_i, y_i)\}_{i=1}^t\) and on the hyperparameters \(\vtheta\), the posterior distribution over function values is also a GP, with posterior mean and covariance given by
\begin{align}
  \mut_t(\vx) &= \mut(\vx) + \kappat(\vx)^\top (\mK^{(\vtheta)} + \sigma_n^2\mI)^{-1}(\vy_t - \mut(\mX_t))\\
  \kt_t(\vx, \vx') &= \kt(\vx, \vx') - \kappat(\vx)^\top (\mK^{(\vtheta)} + \sigma_n^2\mI)^{-1}\kappat(\vx'),
\end{align}
where \(\mK^{(\vtheta)} = [\kt(\vx_i, \vx_j)]_{i=1,j=1}^{t,t} \in \R^{t \times t}\) is the observation covariance matrix, \(\kappat(\vx) = [\kt(\vx, \vx_i)]_{i=1}^{t} \in \R^t\) is the vector of cross-covariances, \(\vy_t = [y_1,\dots,y_t] \in \R^t\) collects the observed objective values, and \(\mX_t\) collects all observed points.
Defining \(\sigma^{2(\vtheta)}_t(\vx) := \kt_t(\vx, \vx)\ \), the marginal posterior at any point \(\vx\) is \(p(f(\vx) \mid \D_t, \vtheta) = \gN(f(\vx) \mid \mut_t(\vx),
\sigma_t^{2(\vtheta)}(\vx))\).
We emphasize that these expressions are conditional on \(\vtheta\); recovering
the surrogate \(p(f \mid \D_t)\) introduced above requires marginalizing the
hyperparameters out, which we discuss in
Section~\ref{sec:fully-bayesian}.

\subsection{Model Averaging}
\label{sec:fully-bayesian}

The posterior in \Cref{sec:gp} is conditional on a particular setting of hyperparameters \(\vtheta\); thus, following the prescription of Bayesian inference, we should marginalize out our uncertainty over it, calculating \(p(f \mid \D_t) = \int p(f \mid \D_t, \vtheta) p(\vtheta \mid \D_t) d\vtheta\).
The problem, however, is that such an integral is intractable.
The hyperparameter posterior \(p(\vtheta \mid \D_t) \propto p(\D_t \mid \vtheta)\, p(\vtheta)\), given a prior \(p(\vtheta)\) and the GP marginal likelihood, is not Gaussian, since \(\vtheta\) enters the likelihood nonlinearly through the kernel matrix, its inverse, and its determinant; moreover, its normalizing constant has no closed form.

As a consequence, practitioners rely on a point estimate \(\smash{\hat{\vtheta}}\) obtained via maximum likelihood or maximum a posteriori.
Such practice is the standard in BO, and is equivalent to approximating \(p(\vtheta \mid \D_t)\) with \(\delta(\vtheta - \smash{\hat{\vtheta}})\).
Unfortunately, this discards all uncertainty about \(\vtheta\), which is especially bad in low-data regimes prevalent in BO.

Monte Carlo (MC) methods offer a solution to this problem, allowing us to approximate the intractable integral as \(p(f \mid \D_t) \approx 1/M \sum_{m=1}^{M} \mathcal{GP}(f \mid \mu_t^{(\vtheta_m)}, k_t^{(\vtheta_m)})\),  where \(\{\vtheta^{(m)}\}_{m=1}^{M} \sim p(\vtheta \mid \D_t)\).
Different Markov chain Monte Carlo (MCMC) algorithms can be used to generate the samples, such as slice sampling~\citep{neal2003slice, murray2010slice} or Hamiltonian Monte Carlo~\citep{duane1987hmc, neal2011hmc} and its adaptive variant NUTS~\citep{hoffman2014nuts}.

Under this mixture surrogate, the marginal acquisition function becomes the corresponding mixture of per-sample acquisitions,
\begin{equation}
  \alpha(\vx; \D_t) = \E_{p(\vtheta \mid \D_t)}[\alpha(\vx; \D_t, \vtheta)] \approx \frac{1}{M}\sum_{m=1}^{M} \alpha(\vx; \D_t, \vtheta^{(m)}),
\end{equation}
provided $\alpha$ is linear in the predictive distribution, so that the marginalization over $\bm\theta$ and the expectation over function values commute.
This holds for expected improvement~\citep[EI;][]{Jones1998EfficientGO}; for its numerically stable variant Log-EI~\citep{ament_unexpected_2024}, which we use in this work, we average in EI space and take the logarithm of the result.
Other acquisition functions can also be applied but may not reduce to a mixture of per-sample acquisitions; for instance, upper confidence bound ~\citep{auer2002ucb,srinivas2010gpucb} is instead computed from the mean and variance of the mixture.

\begin{figure}[!t]
  \centering
  \includegraphics[width=\linewidth]{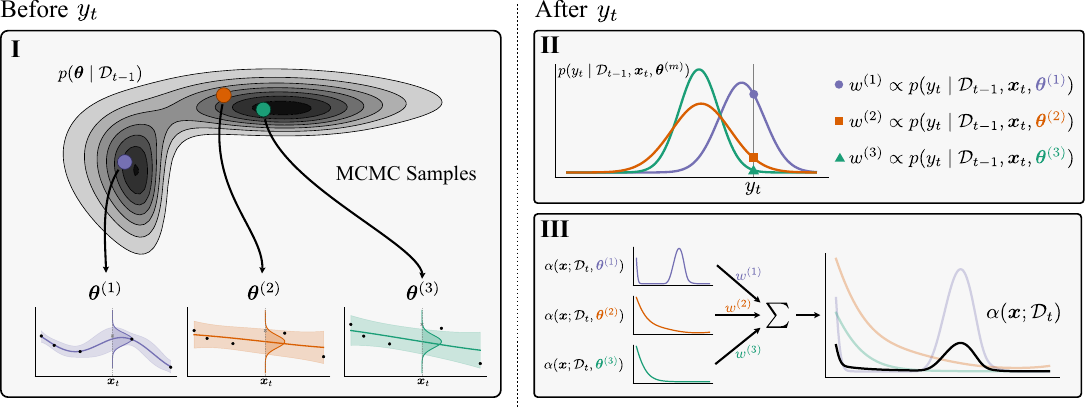}
  \caption{\textbf{Overview of \textsc{ELF-BO}.} Fully Bayesian BO marginalizes the acquisition function over the GP hyperparameters \(\vtheta\), which requires MCMC after every new observation. We instead run it while the objective is still being evaluated. \emph{Left:} \textbf{(I)} with the candidate \(\vx_t\) dispatched and its value \(y_t\) still pending, we sample \(\{\vtheta^{(m)}\}\) from the posterior given the data collected so far, \(p(\vtheta \mid \D_{t-1})\), which is already available. Each sample gives a different GP fit, and hence a different posterior at \(\vx_t\). \emph{Right:} \textbf{(II)} when \(y_t\) arrives, each sample is reweighted by how well it predicted it, \(w^{(m)} \propto p(y_t \mid \D_{t-1}, \vx_t, \vtheta^{(m)})\), and \textbf{(III)} the per-sample acquisition functions \(\alpha(\vx; \D_t, \vtheta^{(m)})\) are combined into the marginal acquisition \(\alpha(\vx; \D_t)\) used to pick the next candidate.}
  \label{fig:viz_abstract}
\end{figure}

\section{Evaluation Latency Fully Bayesian Optimization}
\label{sec:elf}

Consider a dataset \(\D_{t-1}\) has been observed, a point \(\vx_t\) was just dispatched for evaluation, and \(y_t\) is pending.
Our goal is to start the computations to generate the suggestion as follows:
\begin{equation}
  \label{eq:acqf}
  \vx_{t+1} = \argmax_{\vx \in \gX} \E_{\vtheta \sim p(\vtheta \mid \D_t)}[\alpha(\vx;\D_t, \vtheta)],
\end{equation}
where \(\D_t = \D_{t-1} \cup \{(\vx_t, y_t)\}\) and \(\vtheta\) are the GP hyperparameters.
The obvious challenge is that \(y_t\), and consequently \(\D_t\), are not known yet, so we would have to wait for it to run MCMC to obtain samples \(\{\vtheta^{(m)}\}_{m=1}^M\) needed to compute the acquisition.

The central idea of ELF-BO is that we can use the objective evaluation time to obtain samples \(\{\vtheta^{(m)}\}_{m=1}^M\) from the posterior of the previous step \(p(\vtheta \mid \D_{t-1})\) and use a reweighted average to compute \eqref{eq:acqf}, as illustrated in \cref{fig:viz_abstract}.
The following proposition justifies this (proof in Appendix~\ref{apdx:proofs}):

\vspace{0.5em}

\begin{restatable}{proposition}{thmone}
  \label{prop:w}
  Let \(\varphi\) be a function of \(\vtheta\) and \(p(\vtheta \mid \D_t)\) be the posterior of hyperparameters at step \(t \in \{1, \dots, T\}\).
  Suppose that, for each $t \geq 1$, $p(\vtheta \mid \D_t)$ is absolutely continuous w.r.t.\ $p(\vtheta \mid \D_{t-1})$.
  Then, for all $t = 1, \dots, T$,
  \begin{equation*}
    \E_{\vtheta \sim p(\vtheta \mid \D_t)}[\varphi(\vtheta)] = \E_{\vtheta \sim p(\vtheta \mid \D_{t-1})}[w(\vtheta) \, \varphi(\vtheta)],
    \qquad \text{where} \quad
    w(\vtheta) = \frac{p(y_t \mid \D_{t-1}, \vx_t, \vtheta)}{p(y_t \mid \D_{t-1}, \vx_t)}.
  \end{equation*}
\end{restatable}
\begin{proof}[Proof sketch]
  We use the importance sampling argument with the previous posterior as the proposal, so the importance weight is the ratio of the new posterior to the old one.
  By Bayes' rule, this ratio is proportional to the likelihood of the new observation under the predictive distribution.
\end{proof}

\vspace{0.5em}

The preceding result implies that we can start computing the expectation w.r.t.\ the posterior $p(\vtheta \mid \gD_t)$ at the time step $t$ \emph{before} we receive the actual result for $y_t$.
In particular, we can already start sampling $\vtheta$ from $p(\vtheta \mid \gD_{t-1})$ of the previous time step $t-1$.
This is the expensive part of this expectation computation due to expensive MCMC.
Indeed, once $y_t$ is obtained, only simple computation of $w(\vtheta)$, $\varphi(\vtheta)$, and averaging are required.

Notice that \cref{prop:w} casts the expectation as importance sampling.
We can thus use self-normalized importance sampling \citep{mcbook} to approximate \(\E_{\vtheta \sim p(\vtheta \mid \D_t)}[\alpha(\vx;\D_t, \vtheta)]\) using MCMC samples \(\{\vtheta^{(m)}\}_{m=1}^{M}\) from \(p(\vtheta \mid \D_{t-1})\) obtained \emph{before} the evaluation is ready and reweighted \emph{after} \(y_t\) is observed:
\begin{equation}
  \label{eq:reweight}
  \alpha(\vx; \D_t) \approx \sum_{m=1}^{M} w^{(m)} \, \alpha(\vx; \D_t, \vtheta^{(m)}), \qquad w^{(m)} = \frac{p(y_t |\D_{t-1}, \vx_t, \vtheta^{(m)})}{\sum_{m'} p(y_t |\D_{t-1}, \vx_t, \vtheta^{(m')})}.
\end{equation}
This moves all MCMC computations to objective evaluation time. Once the objective value is observed, the acquisition function can be computed by reweighting the sampled hyperparameters based on how well each explains the observed value. \Cref{alg:alg} summarizes the complete procedure.
The reweighting strategy in ELF-BO may resemble sequential Monte Carlo~\citep[SMC;][]{chopin2002sequential}, but in ELF-BO, a fresh set of samples is generated at each function evaluation, whereas SMC propagates a single population across all steps.

Importantly, as in all importance sampling algorithms, the approximation is only reliable when the proposal \(p(\vtheta \mid \D_{t-1})\) adequately covers the target \(p(\vtheta \mid \D_t)\).
This can be monitored through the effective sample size of the self-normalizing weights: \(M_\text{eff} = 1/{\sum_{m=1}^{M} (w^{(m)})^2}\), which is available immediately after \(y_t\) is observed.
Values of \(M_\text{eff}\) near \(M\) indicate that the value of \(y_t\) had high probability under a variety of sampled hyperparameters, so the reweighted samples are representative, while low values indicate that \(y_t\) was poorly explained by most of the samples, meaning a few samples dominate and the estimate will collapse to a single point.
In \cref{sec:analysis} we analyze the effective sample size of the method across optimization steps.

\begin{algorithm}[t]
  \caption{ELF-BO}
  \label{alg:idle-mcmc}
  \begin{algorithmic}
    \Require Initial dataset \(\D_0\), acquisition \(\alpha\), number of samples \(M\), evaluation budget \(T\)
    \Ensure Final dataset \(\D_T\)

    \State{}
    
    \State \(\{\vtheta^{(m)}\}_{m=1}^M \gets \texttt{MCMC}(p(\vtheta \mid \D_0))\)
    \State \(w^{(m)} \gets 1/M\)
    \For{\(t = 1, \dots, T\)}
    \State \(\vx_t = \argmax_{\vx \in \gX} \sum_m w^{(m)} \, \alpha(\vx; \D_{t-1}, \vtheta^{(m)})\)
    \State \(h \gets \mathtt{\mathbf{async} \enspace compute}(f, \vx_t)\) \Comment{evaluation of \(f(\vx_t)\) starts}
    \State {\textcolor{idleblue}{\(\{\vtheta^{(m)}\}_{m=1}^M \gets \texttt{MCMC}(p(\vtheta \mid \D_{t-1}))\)}} \Comment{\textcolor{idleblue}{Computed in parallel}}
    \State \(y_t \gets \mathtt{\mathbf{await} \enspace get\_result} (h)\) \Comment{evaluation of $f(\vx_t)$ ends; obtain \(y_t\)}
    \State \(\D_t = \D_{t-1} \cup \{(\vx_t, y_t)\}\)
    \State Compute \(\{w^{(m)}\}\) as in \eqref{eq:reweight}
    \EndFor
  \end{algorithmic}
  \label{alg:alg}
\end{algorithm}

\subsection{Decision Latency Reduction}
\label{sec:cost}

We now analyze the \emph{decision latency}, the wall-clock time elapsed between an observation arriving and the next query being proposed.
This is relevant whenever idle equipment is costly, samples degrade while waiting, or an operator must dispatch the next experiment immediately, prevalent in automated scientific discovery \citep{MacLeod2020selfdrivinglab}.

Let $T_f$ denote the evaluation latency of the objective, $T_{\text{MC}}$ the time to draw $M$ posterior samples, $T_\alpha$ the time to optimize the expected acquisition function, and $T_w$ the time to compute the weights in \eqref{eq:reweight}.
The decision latencies of standard fully Bayesian optimization and ELF-BO are given by
\begin{equation*}
  L_{\text{FBO}} = T_{\text{MC}} + T_\alpha
  \qquad\text{and}\qquad
  L_{\text{ELF-BO}} = \max(0,\, T_{\text{MC}} - T_f)
  + T_w + T_\alpha ,
  \label{eq:latency}
\end{equation*}
respectively.
Therefore, ELF-BO reduces decision latency by
\begin{equation*}
  \Delta L = L_{\text{FBO}} - L_{\text{ELF-BO}} = \min(T_{\text{MC}}, T_f) - T_w .
\end{equation*}

Note that in many real-world applications, \(T_f \geq T_\text{MC}\), so the entire MCMC sampling procedure runs behind the objective evaluation.
In such cases, the reweighting overhead is negligible in practice.
We validate this analysis in \cref{sec:analysis}.

\section{Related Work}

\paragraph{Fully Bayesian Optimization}
Early work on GPs for global optimization considered a fully Bayesian approach to model hyperparameters~\citep{osborne2009gaussian,benassi2011robust}.
For instance, the seminal work of \citet{snoek2012boml} demonstrated that an MCMC approach using slice sampling~\citep{neal2003slice,murray2010slice} outperforms point estimate-based methods.
More recently, \citet{death2021howbayesian} analyzed the benefits of fully Bayesian methods in BO, demonstrating that the MCMC approach is superior to MAP and variational inference methods when using EI and automatic relevance determination kernels.
In the high-dimensional setting, SAASBO \citep{eriksson2021saas} demonstrated the effectiveness of a fully Bayesian method with NUTS~\citep{hoffman2014nuts} and priors that capture sparseness.
Marginalizing out hyperparameters also provides a textbook solution to overfitting and is beneficial in the overparameterized regime of deep kernels~\citep{ober2021promisesdeepkernel}.
Beyond GP surrogates, fully Bayesian BO has also been performed with Bayesian neural networks~\citep{li2024study,kristiadi_promises_2023,kristiadi_sober_2024}, which could equally benefit from ELF-BO.

\paragraph{Asynchronous Batch Bayesian Optimization}
In this setting, multiple function evaluations run in parallel, and, as soon as one of the function evaluations finishes, a new query is launched without waiting for the pending results.
\citet{ginsbourger2011dealing} was the first to work on this problem, using samples from the predictive posterior at the pending location to calculate the expected EI under different objective realizations.
\citet{alvi2019asynchronous} proposed the use of local penalization around the pending point to reduce redundant selections, while \citet{death2021asynchronous} used an $\varepsilon$-greedy algorithm to achieve the same goal.
Alternatively, \citet{kandasamy2018parallelised} proposed the use of Thompson Sampling~\citep{thompson1933sampling} in which a new sample of the posterior is taken after including only the newly revealed observation.
More recently, \citet{riegler2026standardasync} also demonstrated good results across acquisition functions without explicitly modelling the pending evaluations.
On the related problem of stochastically delayed feedback, \citet{verma2022delayed} considered substituting the minimum function value for observations not yet revealed, a pessimistic version of the Kriging Believer method~\citep{Ginsbourger2010}.
ELF-BO complements this line of work by operating in settings with a single function evaluation at a time, a scenario often found in real-world applications.

\begin{figure}[!t]
  \centering
  \includegraphics[width=\linewidth]{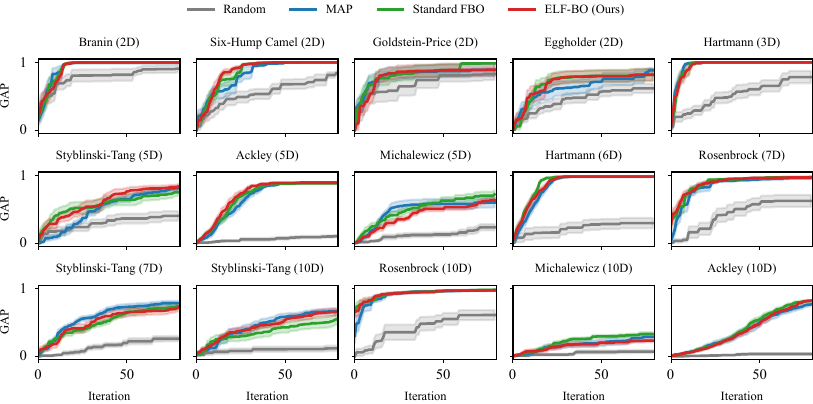}
  \vspace{-1em}
  \caption{\textbf{Standard benchmarks.} Optimization performance of 15 synthetic functions. \elfbo{} consistently follows \mcmc{} on most functions, in many of which the fully Bayesian approaches achieve earlier convergence than \map{}.}
  \label{fig:synth}
\end{figure}

\begin{figure}[!t]
  \centering
  \includegraphics[width=\linewidth]{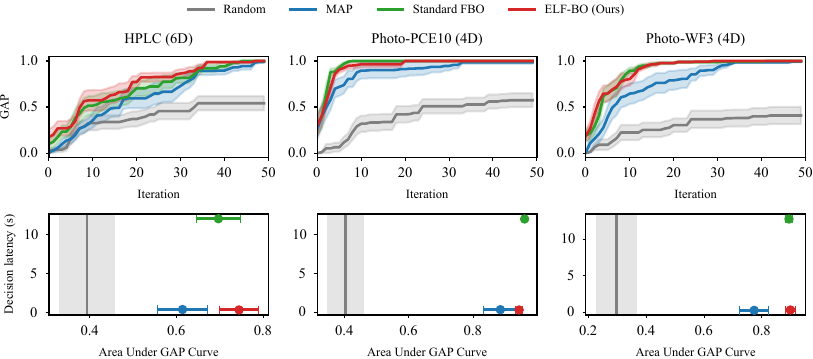}
  \vspace{-1.5em}
  \caption{\textbf{Real-world tasks.} Results on three experimental datasets from the Olympus suite.
  \emph{Top:} best GAP found versus iteration, with standard error shaded. \elfbo{} converges as quickly and reliably as \mcmc{}, while \map{} lags and varies more across seeds.
  \emph{Bottom:} mean decision latency versus area under GAP curve, with error bars showing the standard error over all steps. By sampling while the objective is evaluated, \elfbo{} has area under GAP close to \mcmc{}'s at a latency similar to \map{}'s, more than 10$\times$ lower than \mcmc{}.}
  \label{fig:real}
\end{figure}

\section{Experiments}

We now evaluate the empirical performance of ELF-BO.
Our experiments aim to answer two questions: (i) Does reweighting hyperparameter samples drawn while the oracle is busy match the performance of running MCMC after each observation? (ii) What is the wall-clock time effect of moving sampling under the objective evaluation?

\paragraph{Baselines.}
We compare ELF-BO against random search and two baselines sharing the surrogate and acquisition function, so that differences between methods reflect only how the GP hyperparameters are handled and when the compute is spent.
\random{} search samples uniformly from the domain.
\map{} fits the hyperparameters by maximizing the marginal likelihood plus log prior at every step, which is the default practice.
Standard fully Bayesian optimization (\mcmc{}) runs NUTS on the full dataset after each observation arrives; it is the fully Bayesian gold standard and pays the entire sampling cost between evaluations.
\elfbo{} runs the same NUTS sampler on the data available when a query is dispatched, while the oracle is evaluating, and when the observation arrives, it corrects the draws with a single importance reweighting.

\paragraph{Setup.}
All methods use a GP with a constant mean, homoskedastic Gaussian likelihood, and ARD Matérn-5/2 kernel with output scale.
The lengthscale prior is scaled by dimension as proposed by \citet{hvarfner2024vanilla}.
For the sampling methods, we shift the distribution so that the MAP prior's mode becomes the sampled prior's median, and we reduce the scale from $\sqrt{3}$ to 1.
Unless stated otherwise, Standard FBO and ELF-BO use NUTS~\citep{hoffman2014nuts} with 4 vectorized chains, 512 warmup steps per chain, keeping a total of $M$ samples after thinning by 16 (total of $16M$ samples generated after warmup).
We use a dense mass matrix and a max tree depth of 6.
Log-EI~\citep{ament_unexpected_2024} is used for all methods.
For all experiments, we report the results across 10 trials, with standard errors indicated. 
All methods are implemented in BoTorch~\citep{Balandat2019BoTorchPB}.
Further details on the experimental setup are presented in Appendix~\ref{apdx:exp-details}.

\paragraph{Metrics.}
As different objective functions have different scales, GAP~\citep{barton1984gap} is used to evaluate performance: \(\text{GAP} = (y_i^\text{max} - y_0^\text{max})/(y_* - y_0^\text{max})\), where \(y_i^\text{max}\) is the best observed value at the i-th iteration, \(y_0^\text{max}\) is the highest initial value, and \(y_*\) is the optimal value.
To evaluate the benefits of leveraging objective evaluation time, we also measure the decision latency of each algorithm.

\subsection{Standard Benchmark Functions}

Following \citet{death2021asynchronous}, we use the 15 benchmark functions\footnote{\url{https://www.sfu.ca/~ssurjano/optimization.html}} ranging from $d=2$ to $d=10$: Branin, Eggholder, Goldstein--Price and Six-Hump Camel ($d=2$); Hartmann ($d=3, 6$); Ackley and Michalewicz ($d=5, 10$); Styblinski--Tang ($d=5, 7, 10$); and Rosenbrock ($d=7, 10$).
Each run starts from $2d$ Sobol points and uses a total budget of 100 optimization steps.
The methods that use MCMC take $M=64$ hyperparameter samples per BO step.
We consider noiseless observations.

The results presented in \cref{fig:synth} demonstrate that \elfbo{} tracks \mcmc{} closely on most of the 15 functions, with curves largely overlapping within standard error.
This agreement holds across input dimensions and a range of landscape characteristics, including smooth functions such as Branin and more rugged, highly multimodal functions such as Eggholder.
There are, however, functions for which \elfbo{} and \mcmc{} diverge, such as Michalewicz ($d=5$).
In \Cref{sec:analysis} we analyze this failure through the lens of effective sample size.
Finally, we observe that \elfbo{} follows \mcmc{} both when it outperforms \map{} and when they are equivalent, showing no clear preference for cases where point estimates are ``good enough''.
This matches the goal of the work, which is to provide a small-latency alternative to fully Bayesian optimization.

\begin{figure}[!t]
  \centering
  \includegraphics[width=\linewidth]{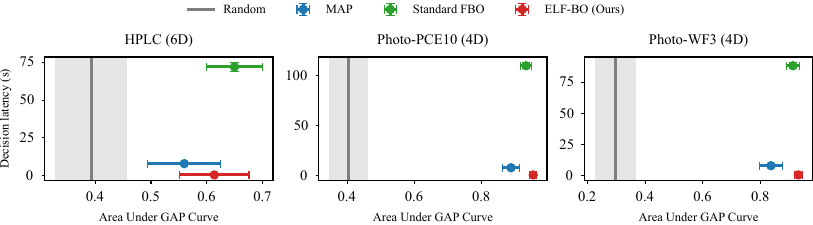}
  \vspace{-1.5em}
  \caption{\textbf{Deep kernel on real-world tasks.} Results of the deep kernel on Olympus datasets. As in experiments with Matérn kernel, \elfbo{} achieves Pareto optimality over \mcmc{} and \map{}, with an even more pronounced reduction in decision latency. See Appendix~\ref{apdx:additional-res} for GAP curves.}
  \label{fig:deep}
\end{figure}

\subsection{Real-World Applications}
\label{sec:real}

In addition, we evaluate ELF-BO on three experimental datasets curated in the Olympus benchmark suite~\citep{hase_olympus_2021}.
Rather than using learned emulators, we directly treat the datasets as a virtual library---thus, querying a point returns its recorded experimental measure.
HPLC ($d=6$ and $|\gX|=1{,}386$) tunes the sampling sequence of a robotic direct-inject HPLC system to maximize peak area \citep{roch2020chemos}.
Photo-PCE10 and Photo-WF3 ($d=4$ and $|\gX|=1{,}040$ each) optimize the composition of polymer blends to minimize\footnote{We negate the objective and consider it a maximization problem.} photodegradation \citep{langner2020beyond}.
We choose the three Olympus datasets with more than 1000 data points to provide a large enough library for BO.
GAP is computed with $y_*$ set to the best measurement in each dataset.

All three datasets were collected through physical experiments.
In a live campaign where Standard FBO chooses each experiment before it is run, every evaluation would take far longer than sampling the hyperparameter posterior, leaving the planner idle for most of the campaign---precisely the setting ELF-BO is designed for.
Although the photodegradation data was collected in a high-throughput setting, we optimize all three tasks sequentially.
We use an initial dataset of \(2d\) Sobol points and a total budget of 50 steps for all tasks.
Standard FBO and ELF-BO use $M=128$ hyperparameter samples per BO step.
More details on the tasks and datasets are provided in Appendix~\ref{apdx:exp-details}.

The optimization curves presented in \cref{fig:real} (top) demonstrate that, across all tasks, \elfbo{} closely tracks \mcmc{} in performance.
The results also highlight that in real-world applications with limited evaluation budget, a fully Bayesian approach can improve sample efficiency, achieving better results in earlier optimization steps.
We analyze the sample efficiency using the area under the GAP curve, which is presented in \cref{fig:real} (bottom) alongside average decision latency.
This demonstrates the Pareto dominance of \elfbo{} over \mcmc{} and \map{}, displaying optimization performance on par with \mcmc{} while providing suggestions more than $10\times$ faster.

\paragraph{Deep kernel.}
A natural use case for a fully Bayesian approach is with deep kernels~\citep{wilson2016deep}, as marginalizing out the network weights can mitigate the overfitting of point estimates in the overparameterized regime~\citep{ober2021promisesdeepkernel}.
In addition, they offer a challenging setting for \elfbo{}, as the sampled parameters include neural-network weights, a step toward BNN surrogates.
We repeat the real-world benchmarks, replacing the Matérn kernel with a deep kernel (two-layer neural network; see Appendix~\ref{apdx:exp-details}).
As shown in \cref{fig:deep}, the decision latency of \mcmc{} grows 5--9$\times$ relative to Matérn, whereas \elfbo{} remains under 1\,s, a 70--110$\times$ reduction, while remaining competitive in sample efficiency.

\section{Analysis}
\label{sec:analysis}

We now analyze ELF-BO under different settings, improving our understanding of the method.
To this end, we study the scalability with respect to dataset size and the impact of effective sample size.

\paragraph{Scaling with the number of observations.}
The cost of fully Bayesian optimization grows with the number of observations \(n\).
However, ELF-BO is unaffected by this growth as long as sampling fits within the evaluation time.
\Cref{fig:scaling} shows the decision latency as a function of the number of observations on Photo-WF3 ($d = 4$) for the deep kernel model in \cref{sec:real}.
While \mcmc{} latency grows up to 24 minutes, \elfbo{} remains nearly constant at under 3 seconds.
This highlights the large gains that can be obtained by ELF-BO when optimization campaigns are long and use more complex models.
The same figure indicates the evaluation time required to fully hide sampling at each step, providing a practical guide for when ELF-BO is applicable.

\paragraph{Impact of effective sample size.}
As noted in \Cref{sec:elf}, ELF-BO depends on how representative the samples collected during the objective evaluation are.
We analyze \(M_\text{eff}\) of Michalewicz ($d = 5$), a function for which ELF-BO has difficulties tracking Standard FBO performance, and Hartmann ($d = 6$), a function close in dimensionality for which ELF-BO performs well.
\Cref{fig:ess} presents the step-wise fraction of effective samples, defined as \(M_\text{eff}/M\).
The effective sample size for Michalewicz ($d = 5$) is lower for most of the optimization steps, although it remains far from collapse for both functions.
The lower values of \(M_\text{eff}\) for Michalewicz ($d = 5$) help explain why ELF-BO performs worse on that function.
Since \(M_\text{eff}\) is available as soon as \(y_t\) is observed, it provides a cheap diagnostic for flagging steps where the reweighted estimate may be less reliable.

\begin{figure}[!t] 
\centering
\begin{minipage}[t]{0.48\linewidth} 
\centering
\includegraphics[width=\linewidth]{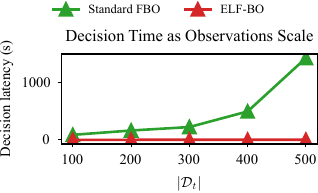}
\vspace{-1.5em}
\caption{\textbf{Deep kernel on larger datasets.} Decision latency on Photo-WF3 ($d = 4$) as observation set grows. \mcmc{} increases with dataset size while \elfbo{} remains constant.}
\label{fig:scaling}
\end{minipage}\hfill 
\begin{minipage}[t]{0.48\linewidth}
\centering
\includegraphics[width=\linewidth]{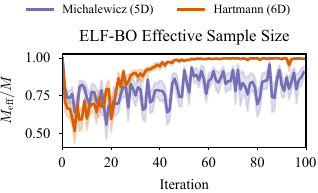}
\vspace{-1.5em} 
\caption{\textbf{Normalized effective sample size per iteration.} Michalewicz ($d = 5$) exhibits a lower effective sample size than Hartmann ($d = 6$) for most of the optimization steps.} 
\label{fig:ess} 
\end{minipage}
\end{figure}
\section{Conclusion}

We introduced ELF-BO, a fully Bayesian optimization algorithm that leverages the objective's execution latency to alleviate the time taken to make a decision on the next point to query.
By drawing samples from the posterior given the data available at dispatch time and reweighting them once the observation is revealed, ELF-BO removes MCMC from the time between function evaluations.
Across 15 synthetic functions and three real-world experimental datasets, ELF-BO matches the optimization performance of MCMC-based fully Bayesian optimization on most tasks while keeping decision latency close to point estimate BO times.

\paragraph{Limitations and Future Work}
ELF-BO's latency gains are observed in applications for which objective function evaluations take more time than MCMC, but are absent in applications where evaluations are near-instant.
Moreover, we only consider GP surrogates in the sequential setting.
Natural extensions include BNN surrogates and asynchronous batch BO with multiple pending evaluations.
More broadly, ELF-BO shows that the planner's idle time, which is inherent to expensive black-box optimization, can be spent on better decisions.

\bibliography{bibliography}
\bibliographystyle{iclr2027_conference}

\clearpage
\appendix
\section{Proofs}
\label{apdx:proofs}
\thmone*

\begin{proof}
  Let $t \in \{ 1, \dots, T \}$ be arbitrary.
  By the importance sampling argument:
  \begin{equation*}
    \begin{split}
      \E_{\vtheta \sim p(\vtheta \mid \D_t)}[\varphi(\vtheta)] &= \int \varphi(\vtheta) \, p(\vtheta \mid \D_t) \, d\vtheta \\
      &= \int \varphi(\vtheta) \underbrace{\frac{p(\vtheta \mid \D_t)}{p(\vtheta \mid \D_{t-1})}}_{=:w(\vtheta)} p(\vtheta \mid \D_{t-1}) \, d\vtheta \\
      &= \E_{\vtheta \sim p(\vtheta \mid \D_{t-1})}[w(\vtheta) \,\varphi(\vtheta)] \, .
    \end{split}
  \end{equation*}
  In particular, the importance weight is $w(\vtheta) = \frac{p(\vtheta \mid \gD_t)}{p(\vtheta \mid \gD_{t-1})}$.
  
  Recall that $\gD_t = \gD_{t-1} \cup \{ (\vx_t, y_t) \}$. 
  Thus $p(\vtheta \mid \gD_t) = p(\vtheta \mid \gD_{t-1}, \vx_t, y_t)$.
  Therefore, by Bayes' rule we have that:
  \begin{equation*}
    \begin{split}
      p(\vtheta \mid \D_t) &= p(\vtheta \mid \D_{t-1}, \vx_t, y_t) \\
      &= \frac{p(y_t \mid \D_{t-1}, \vx_t, \vtheta) \, p(\vtheta \mid \D_{t-1}, \vx_t)}{p(y_t \mid \D_{t-1}, \vx_t)} \\
      &= \frac{p(y_t \mid \D_{t-1}, \vx_t, \vtheta) \, p(\vtheta \mid \D_{t-1})}{p(y_t \mid \D_{t-1}, \vx_t)} \, .
    \end{split}
  \end{equation*}
  Therefore, plugging this into \(w(\vtheta)\) we get:
  \begin{equation*}
    \begin{split}
      w(\vtheta) &= \frac{p(\vtheta \mid \D_t)}{p(\vtheta \mid \D_{t-1})} \\
      &= \frac{p(y_t \mid \D_{t-1}, \vx_t, \vtheta) \, p(\vtheta \mid \D_{t-1})}{p(y_t \mid \D_{t-1}, \vx_t)} \frac{1}{p(\vtheta \mid \D_{t-1})} \\
      &= \frac{p(y_t \mid \D_{t-1}, \vx_t, \vtheta)}{p(y_t \mid \D_{t-1}, \vx_t)} \\
      &\propto {p(y_t \mid \D_{t-1}, \vx_t, \vtheta)},
    \end{split}
  \end{equation*}
  where the last proportionality is because the denominator is constant w.r.t.\ $\vtheta$.
\end{proof}

\section{Experimental Details}
\label{apdx:exp-details}

\subsection{Real-World Applications}

\paragraph{HPLC.} 
The dataset originates from the autonomous calibration of a robotic sampling sequence for direct-inject HPLC analysis~\citep{roch2020chemos}.
The authors report sessions of 100 autonomous experiments over approximately seven hours, corresponding to roughly four minutes per evaluation.

\paragraph{Photo-PCE10 and Photo-WF3.} These datasets report the photostability of photovoltaic blends~\citep{langner2020beyond}.
The experimental time reported by the authors is 18 hours for a batch of samples.
As mentioned in the main text, we consider the problem in the sequential (non-batch) optimization setting.

\subsection{Deep Kernel Details}

The deep kernel is implemented as \(\sigma_f \cdot k(\gE(\vx), \gE(\vx'))\), where \(k\) is a Matérn kernel with lengthscales fixed at 1, \(\sigma_f\) is a kernel outputscale, and  \(\gE : \R^d \to \R^8\) is a neural network with a single hidden layer of dimension 16 and \(\tanh\) activations.
Taking into consideration the observation noise and constant mean, the deep kernel has a total of \(16d + 155\) hyperparameters.
To give NUTS a proper posterior to sample and prevent overfitting in MAP, we impose a zero-mean normal prior with variance 0.25 for all the neural network's biases and $4/L_{in}$ for its weights, where $L_{in}$ is the dimension of the layer's input.

For the experiments in \cref{sec:real}, the MAP deep kernel is optimized using 500 iterations of Adam~\citep{kingma2015adam} per restart, with 4 restarts and a learning rate of 0.01.
For Standard FBO and ELF-BO, samples are obtained using NUTS with a diagonal mass matrix, performing 512 warmup steps and 4096 samples, resulting in \(M=256\) after thinning by 16, with a max tree depth of 10.
For the scaling experiments in \cref{sec:analysis} we use \(M=128\) and the Log-EI is evaluated on batches of 32 candidates at a time.

\subsection{Computational Resources}
Experiments are implemented in Python using the BoTorch \citep{Balandat2019BoTorchPB} library version 0.18.1.
All synthetic benchmarks are executed on a single NVIDIA L40S, while real-world experiments make use of an NVIDIA H100.

\subsection{Further Experimental Details}

A throwaway NUTS run happens before the first BO step, so that device start-up is not considered in latency results.

\paragraph{Acquisition function optimization}
For the continuous synthetic functions, we use 10 restarts, with 512 raw samples. For the discrete real-world problems, every unused row is scored, 512 rows at a time.
\section{Additional Results}
\label{apdx:additional-res}

\begin{figure}[!t]
  \centering
  \includegraphics[width=\linewidth]{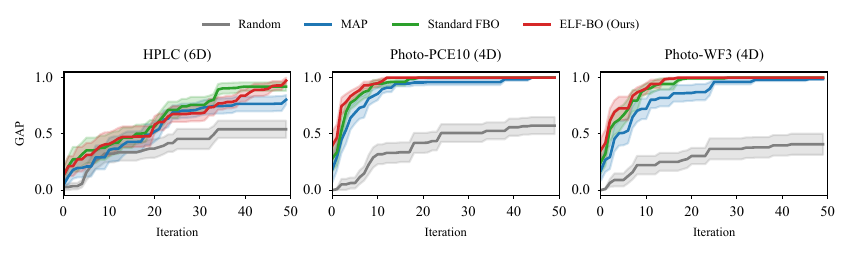}
  \vspace{-1.5em}
  \caption{\textbf{Deep kernel on real-world tasks.} Optimization curves of the deep kernel methods on Olympus datasets.}
  \label{fig:deep-gap}
\end{figure}

\paragraph{Deep Kernel} \Cref{fig:deep-gap} presents the complete GAP curves for the deep kernel experiments described in \cref{sec:real}.

\end{document}